\documentclass[12pt]{article}
\usepackage[utf8]{inputenc}
\usepackage{geometry}
\usepackage[scaled]{helvet}
\usepackage{amsmath,amssymb,amsthm}
\usepackage{algorithm}
\usepackage{algpseudocode}
\usepackage{graphicx}
\usepackage[colorlinks=true,linkcolor=black,citecolor=black,urlcolor=black]{hyperref}
\usepackage{cleveref}
\usepackage{booktabs}
\usepackage{bbm}

\newtheorem{theorem}{Theorem}
\newtheorem{lemma}[theorem]{Lemma}
\newtheorem{corollary}[theorem]{Corollary}

\title{Memory makes computation universal, remember?}
\author{Erik Garrison\\
  \texttt{egarris5@uthsc.edu}\\[1ex]
  }
\date{\today}

\begin{document}
\maketitle

\begin{abstract}
Recent breakthroughs in AI capability have been attributed to increasingly sophisticated architectures and alignment techniques, but a simpler principle may explain these advances: memory makes computation universal.
Memory enables universal computation through two fundamental capabilities: recursive state maintenance and reliable history access.
We formally prove these requirements are both necessary and sufficient for universal computation.
This principle manifests across scales, from cellular computation to neural networks to language models.
Complex behavior emerges not from sophisticated processing units but from maintaining and accessing state across time.
We demonstrate how parallel systems like neural networks achieve universal computation despite limitations in their basic units by maintaining state across iterations.
This theoretical framework reveals a universal pattern: computational advances consistently emerge from enhanced abilities to maintain and access state rather than from more complex basic operations.
Our analysis unifies understanding of computation across biological systems, artificial intelligence, and human cognition, reminding us that humanity's own computational capabilities have evolved in step with our technical ability to remember through oral traditions, writing, and now computing.
\end{abstract}

\section{Introduction}

We are collections of cells with interwoven systems of memory.
The genome maintains computational state through an unbroken chain of survival, while molecular markers preserve cell identity, immune systems learn from encounters, and neural architectures give rise to consciousness.
Through memory, we became beings that could understand our own existence.
Our species' greatest evolutionary leap came not through biological adaptation but through cultural innovation---the ability to maintain and pass on our thoughts beyond individual lifespans through oral traditions, writing, and libraries.
Each technology expanded our ability to preserve and process information across time, transforming human civilization through new forms of memory.
This cultural evolution of memory systems mirrors the fundamental role of memory in all computation, from cellular mechanisms to artificial intelligence.

When we finally mastered computation, we did it by capturing lightning in rock: precise control of electron flows in crystalline structures.
Yet as we strive to create systems that can understand and generate human language, we find ourselves returning to biological patterns: massive parallel networks trained simultaneously across countless processors.
These parallel training constraints lead our artificial systems to mirror our own biological limitations---they too are restricted to threshold computations, like our ``blobby, wet'' molecular mechanisms \cite{alberts2022molecular}.
Their intelligence emerges not from increasingly sophisticated components but from maintaining and accessing state across time, just as ours did.

Memory is not merely a feature of computational systems: it is the key capability that enables universal computation~\cite{turing1936computable}.
Even systems with very basic state transitions can achieve universal computation when equipped with memory~\cite{merrill2023parallelism,peng2024limitations}.
This principle is demonstrated by remarkably minimal universal systems like Rule 110 cellular automata~\cite{cook2004universality} and the universality of the mov instruction set~\cite{dolan2013mov}.
We show how this fundamental insight manifests in computational systems, from neural networks to biological cells~\cite{wang2023parallel}, where simple parallel operations combined with state maintenance enable complex computation~\cite{swamy1983space,boyle2024memory}.

We demonstrate how two basic capabilities---recursive state maintenance and reliable history access---are sufficient for universal computation \cite{turing1936computable,sipser1996introduction,savage1994space,bennett1989time}.
We first establish this formally through a constructive proof, grounding our subsequent analysis in computational theory.
We then examine how parallel processing limitations manifest in both artificial neural networks and biological systems, showing how these systems achieve complex computation through memory rather than sophisticated processing units.
Finally, we explore the practical implications for artificial intelligence development, particularly in the context of large language models and recursive reasoning.
This analysis provides a unified perspective on computation across domains while suggesting new directions for AI architecture development focused on robust state maintenance mechanisms.

\section{Memory enables universal computation}

A system has recursive state processing if it can read its current state $s$, generate an update function $f(s)$ that produces the next state, and maintain the new state $f(s)$ for subsequent processing.
This state maintenance must be robust against errors and degradation.

A system has reliable history access if it can: (1) reference any previous state $s_i$ from its computation history without error, (2) distinguish between different states in its history unambiguously, (3) maintain and determine the correct temporal order of states, establishing clear boundaries between computational periods, and (4) guarantee the integrity of stored states. Note that standard RNNs are capacity-limited and may require external memory modules or augmented architectures to meet these requirements fully.

These capabilities enable universal computation through a straightforward construction \cite{sipser1996introduction,deutsch1995universality,bennett1989time}.
Given a Universal Turing Machine $U$, we implement it as follows: The state $s = (q, p, a)$ encodes the current machine state $q$, the position $p$ of the tape head, and the symbol $a$ under the head.
At each step, the system processes current state $s$ to determine $(q', a', d) = \delta(q, a)$, where $q'$ is the next machine state, $a'$ is the symbol to write, and $d$ is the head movement.
It uses history access to reconstruct tape contents as needed and maintains the new state for the next step.

We show how these two requirements enable universal computation through a constructive proof that proceeds in three steps: encoding the machine state, simulating tape operations, and proving correctness with complexity bounds.

\vspace{1em}

\begin{theorem}[Universality]
Any system $S$ with recursive state maintenance and reliable history access can simulate a Universal Turing Machine with at most logarithmic overhead in space and time complexity \cite{boyle2024memory,liskiewicz1994complexity}.
\end{theorem}

\textit{Remark.}
This theorem draws on a classic line of work showing that even small instruction sets can achieve Turing-completeness \cite{savage1994space,deutsch1995universality,bennett1989time,dolan2013mov}.
Our formulation recasts these results in a biological and chain-of-thought context, emphasizing that parallel systems with minimal updates are still universal once they reliably maintain and retrieve computational state.

\begin{proof}
We proceed in three steps:

\vspace{0.5em}
\noindent\underline{First}, we show that $S$ can implement the basic operations of a UTM. Let $M$ be a UTM with state set $Q$, tape alphabet $\Gamma$, and transition function $\delta$. We construct a simulation in $S$ as follows:

\vspace{0.5em}
\noindent\textbf{State Encoding:} Each configuration of $M$ is encoded as a tuple $s = (q, p, a, t)$ where: $q \in Q$ is the current machine state, $p \in \mathbb{N}$ is the head position, $a \in \Gamma$ is the symbol under the head, and $t \in \mathbb{N}$ is the step counter.

\vspace{0.5em}
\noindent\textbf{Tape Simulation:} For each position $p$ visited by $M$, we maintain a history entry $h_p = (p, a_p, t_p)$ where $a_p \in \Gamma$ is the symbol written at position $p$ and $t_p$ is the time step when this symbol was written.

\vspace{0.5em}
\noindent\underline{Second}, we prove that these operations preserve computational state correctly through the following invariants:

\begin{lemma}[State Coherence]
At each step $t$, the simulated configuration $(q, p, a)$ exactly matches the configuration of $M$ after $t$ steps on the same input.
\end{lemma}

\begin{proof}
By induction on $t$:

\vspace{0.3em}
\noindent\textit{Base case} ($t=0$): The initial configuration matches by construction.

\vspace{0.3em}
\noindent\textit{Inductive step}: Assume the invariant holds for step $t$. 
For step $t+1$, $S$ reads current state $(q, p, a, t)$, computes $(q', a', d) = \delta(q, a)$, updates position $p' = p + d$, uses history access to find $a''$ at $p'$, and maintains new state $(q', p', a'', t+1)$.
This exactly mirrors $M$'s transition function, preserving the invariant.
\end{proof}

\begin{lemma}[History Consistency]
For any position $p$ and time $t$, the symbol recorded in history matches what would be on $M$'s tape at position $p$ at time $t$.
\end{lemma}

\begin{proof}
We maintain two invariants: each write operation creates a history entry with the current time step, and when reading position $p$, we retrieve the entry with maximum $t_p \leq t$.
This ensures we always see the most recent write to each position, exactly matching $M$'s tape contents.
\end{proof}

\vspace{0.5em}
\noindent\underline{Third, and finally}, we analyze the complexity bounds: The simulation incurs logarithmic overhead in both space ($O(\log t)$ bits for step counter and $O(\log n)$ bits for position encoding) and time ($O(\log t)$ for history access operations). These bounds are tight \cite{parzych2024memory,hhan2024new,boyle2024memory}, as demonstrated through recent impossibility results. Therefore, $S$ simulates $M$ with logarithmic overhead in both space and time \cite{savage1994space,vonkorff2019molecular,bennett1989time}.
\end{proof}

\begin{corollary}[Universality]
Any system with recursive state maintenance and reliable history access is Turing-complete.
\end{corollary}

\vspace{1em}

Consider a basic threshold unit that checks if a number is greater than zero.
Without memory, this unit can only output yes/no signals---a trivial, static operation.
However, with memory to store previous results, this same unit becomes capable of counting: starting from zero, it can repeatedly increment its input and maintain the running total, transforming a simple parallel threshold operation into a sequential counter.
This counter then serves as a fundamental building block for more complex operations, enabling program sequencing and even Turing machine simulation.

Reliable history access requires systems to maintain clear boundaries between computational states---a requirement that manifests as temporal structure in physical implementations.
In biological systems, though epigenetics and threshold signals provide a partial form of memory, they do not fulfill the indefinite, unambiguous addressability required for our definition of reliable history access.
In neural networks, this structure emerges through attention windows that establish sequence order and layered processing that maintains state boundaries \cite{martini2015information,quentin2019differential}.
Biological systems show similar patterns: cells partition computation through cell cycle phases and stable epigenetic states \cite{bruno2022epigenetic}, while developing organisms maintain state through precisely timed developmental stages and morphology \cite{turing1952chemical}.
Despite their diverse implementations, all these systems share a common need to establish unambiguous boundaries between computational states, enabling reliable access to their processing history.
Each of these domains confirms the same fundamental principle: partial or ephemeral memory alone cannot sustain universal computation; systems must meet conditions for reliable history access to achieve full Turing-completeness.

\section{Memory constraints in parallel systems}

Parallel constraints arise in real-world systems, forcing reliance on memory for complex computation.
Both artificial and biological systems face fundamental constraints in parallel computation that shape their architecture and capabilities.
These constraints emerge from different sources but lead to remarkably similar solutions through memory and state maintenance.

Neural architectures face a fundamental computational barrier: they are restricted to $\text{TC}_0$ complexity, as proven by recent theoretical work \cite{merrill2023parallelism,peng2024limitations}.
This limitation stems directly from the requirements of parallel training at scale.
To train efficiently across multiple devices, neural networks must permit independent parameter updates across different portions of the network \cite{barrett2019analyzing}.
This independence requirement manifests in different ways: transformers treat each position independently modulo attention weights, convolutional networks rely on spatial locality, and even recurrent networks---despite being Turing-complete in theory under strictly sequential operation---must make independence assumptions between time steps when unrolled for parallel training \cite{bengio2017deep}.
Even recent architectural innovations like structured state spaces and their variants re-parameterize sequential computation for parallel training but still operate under practical concurrency constraints \cite{gu2021efficiently,gu2023mamba,dao2024transformers,nguyen2024hyenadna}.

Recent empirical work provides striking validation of this principle.
Feng et al. \cite{feng2024rnns} demonstrated that RNNs stripped to their minimal state maintenance operations, but reorganized for parallel processing, match the performance of recent sequence models.
This shows how basic state maintenance enables complex computation.

Similarly, biological systems face parallel processing constraints due to their physical nature.
The central dogma of molecular biology reveals how cells solve computation through inherently parallel mechanisms \cite{wang2023parallel,cai2024efficient,fu2023scgrn}.
Unable to engineer precise electromagnetic flows like silicon systems, cells must process information through vast networks of simultaneous molecular interactions.
Every aspect of cellular computation occurs through parallel reactions governed by concentration thresholds and binding affinities \cite{alberts2022molecular}.
A typical eukaryotic cell contains gigabase-pairs of DNA, comparable quantities of RNA, and billions of proteins \cite{milo2013protein}.
These molecules are in constant motion, with proteins tumbling at gigahertz frequencies \cite{zhang2023molecular}, vibrating in the megahertz to terahertz range, and diffusing across the entire cell within seconds.
In this dense molecular soup, every component is practically interacting with every other component all the time, making true sequential processing physically impossible.
The sheer scale of these simultaneous interactions—billions of molecules moving at gigahertz frequencies within the confined space of a cell—forces computation to be inherently parallel.

Additional memory mechanisms—like recurrent connections and attention—enable sequential logic in parallel networks.
Each additional processing step increases computational capability in a measurable way, but only when proper state maintenance is preserved across steps.
Recent work shows this directly supports our framework by demonstrating how memory mechanisms, not architectural complexity, determine computational power (for details of positional embeddings and attention-based ordering, see \cite{merrill2023parallelism,merrill2024}).
These mechanisms inject sequential capability into parallel architectures not through more complex units, but by maintaining and accessing state across processing steps.

DNA-based record-keeping, epigenetics, and neural time cells exemplify how evolution converged on memory to enable sequential processing.
The CRISPR-Cas system represents perhaps the most elegant implementation, maintaining a chronological record of viral infections \cite{kim2019crispr} through stable DNA-based storage combined with active protein-based processing \cite{sadremomtaz2023digital}.
In neural circuits, specialized "time cells" fire in precise sequences while ramping cells modulate their activity to provide temporal context \cite{quentin2019differential}, with the prefrontal cortex and parahippocampal regions integrating contextual details \cite{martini2015information}.
These mechanisms achieve ordered processing not through sequential computation, but by maintaining state through multiple parallel components orchestrated in time.

Ultimately, enhanced memory—across domains—provides the cornerstone for circumventing parallel limitations.
The anterior-posterior axis and segmentation patterns in animal embryos emerge through maintained morphogen gradients and sequential state updates \cite{pastor2020computation}.
These patterns demonstrate remarkable robustness, with the ability to regenerate proper proportions even after significant tissue removal or reorganization \cite{lobo2012towards}.
This principle was first formalized by Turing's reaction-diffusion theory \cite{turing1952chemical}, showing how simple molecular interactions can implement sophisticated state maintenance through physical structure.
Through the interaction of fast-diffusing inhibitors with self-enhancing activators, organisms create stable spatial patterns that maintain developmental state information across time.

These systems share a common principle of memory-enabled computation.
Recent work in synthetic biology validates this framework, successfully implementing recording systems that combine DNA-based storage with protein-based processing \cite{sheth2017multiplex}.
At the cellular level, even apparently sequential processes like the cell cycle are implemented through parallel molecular networks with threshold-based transitions rather than true sequential computation \cite{alberts2022molecular}.
This mirrors how artificial systems overcome parallel processing limitations through recursive state maintenance rather than more complex basic units.

\section{Practical implications}

The lens of recursive state maintenance reveals fundamental patterns in how language models achieve complex reasoning \cite{dickson2024trust,ahn2024recursive,openai2024o1}.
These systems face inherent parallel processing limitations that constrain their computational capabilities \cite{merrill2023parallelism}.
Recent evidence demonstrates both the power and limitations of chain-of-thought approaches in overcoming these constraints \cite{liu2024mind}.
Traditional scaling approaches focused on increasing parallel processing power through larger models and more sophisticated attention mechanisms \cite{shallue2019measuring}.
However, our analysis shows this approach alone cannot overcome the fundamental limitations in sequential reasoning \cite{peng2024limitations}.
Language models achieve sequential reasoning by learning to replicate computational patterns embedded in human language.
Through exposure to text containing human reasoning traces, these models internalize the basic transition functions that characterize structured thought.
When provided with sufficient context as state, they can apply these learned patterns to perform step-by-step reasoning.
This explains why chain-of-thought prompting proves effective: it provides a framework for maintaining computational state across multiple steps.

The ARC Prize competition provides compelling evidence of the limitations inherent in scaling models alone.
While advanced LLMs with sophisticated architectures reach only 5-14\% accuracy through direct prompting, introducing state maintenance through test-time training boosts performance to 55.5\% \cite{chollet2024arc}.
This dramatic improvement aligns with our framework's prediction: parallel processing alone cannot overcome sequential reasoning constraints without reliable state maintenance, just as biological systems do not solve sequential tasks by making cells more complex \cite{wang2023parallel}.

OpenAI's o3 model further reinforces the crucial role of memory and state maintenance.
The system achieves 82.8\% accuracy on 400 public tasks of ARC-AGI-1 in high-efficiency mode (costing \$6{,}677 and using roughly 33M tokens) and 91.5\% in low-efficiency mode (costing approximately \$1.15M and using around 9.5B tokens), while retaining high performance on 100 held-out tasks \cite{chollet2024o3}.
Our framework suggests that what o3's authors term ``deliberative alignment'' \cite{openai2024o3blog,guan2024deliberative} may be more precisely understood as an emergent property of achieving reliable computation through memory and state maintenance, building on the same principles that drive universal computation in biological and artificial systems alike.
Since a system without reliable state maintenance cannot consistently execute any computational pattern---including following policies---the o3 system compensates by searching over potential computational paths, using policy adherence as a way to detect which paths successfully maintain coherent computational state \cite{chollet2024o3, openai2024learning}.
In other words, reliable policy following and reliable computation are fundamentally inseparable: a system capable of one must be capable of the other.
The substantial computational costs---\$17--20 per task in high-efficiency mode---likely reflect both the quadratic overhead of transformer attention and the need to evaluate multiple possible computational traces.
This suggests o3 represents something qualitatively different from traditional LLMs \cite{willison2024o3}---by achieving reliable state maintenance through parallel search and selection, albeit at significant cost, it appears to meet the requirements for universal computation outlined in our framework.

o3's chain-of-thought fulfills both ``recursive state maintenance'' and ``reliable history access,'' storing partial results (including policy references) and retrieving them at each step.
The same memory mechanism that preserves logical continuity across iterations also ensures continuous policy adherence, preventing the system from losing alignment constraints in the midst of multi-step reasoning.
By unifying these capacities, o3 illustrates how reliable computation and faithful policy following naturally converge when robust memory mechanisms are in place.
While o3 invests heavily in chain-of-thought (``state-keeping'') and robust policy referencing, it still does not unify memory in a single, universal mechanism.
In practice, it combines multiple partial strategies---chunked or rolling contexts, local chain-of-thought, and, at times, external retrieval systems---to maintain its internal state.
Hence, o3 should be seen as an early illustration of how emphasizing memory-like processes can push models beyond purely pattern-based reasoning, rather than a complete or general-purpose memory system.
The system's impressive performance may emerge not from sophisticated ethical reasoning, but from its ability to maintain consistent computational state while searching for valid solutions that follow learned patterns.

Instead, architectural innovation should focus on robust mechanisms for state maintenance and recursive processing \cite{jung2020new}.
This might involve memory systems that maintain coherent state across multiple processing steps \cite{zhu2024overcoming}.
Recent research has identified specific conditions where chain-of-thought reasoning can actually reduce performance, particularly in tasks where the underlying computation requires implicit pattern recognition rather than explicit reasoning \cite{liu2024mind}.
This aligns with our framework's prediction that memory-based recursive computation is most effective when the task requires true sequential dependence, rather than parallel pattern matching.
State verification mechanisms that ensure consistency across recursive processing steps, similar to how biological systems maintain state coherence through specialized cellular structures \cite{espinosa2024molecular}, become crucial for maintaining reliable computation.

This framework highlights two complementary modes of computation in language models \cite{wei2022chain}.
Pattern matching relies on parallel processing to recognize textual features in a single pass, while multi-step reasoning depends on robust state maintenance.
Without memory, a model is confined to parallel inference and cannot harness sequential logic.
In contrast, chain-of-thought or memory-based methods accumulate partial results step by step, enabling more complex computations.
However, explicitly spelling out each inference can sometimes degrade performance on tasks better suited to implicit pattern recognition \cite{liu2024mind}, mirroring how conscious reasoning can undermine human intuition.

Further insight comes from the relationship between compression and computation \cite{dickson2024trust}.
Language models compress static patterns into distributed representations, yet require explicit prompting for sequential tasks because compression alone does not support dynamic state updates.
Herein lies the power of computational history access: each reasoning step builds on previous results while maintaining temporal structure \cite{wei2022chain}, providing the essential scaffolding for advanced, memory-driven computations.

The remarkable leap in language model capabilities emerged from massive overtraining on an unprecedented coverage of human knowledge \cite{schuurmans2024autoregressive}.
This represents a form of transfer learning where human computational patterns are distilled into the model's parameters.
This extensive exposure enabled these systems to learn and internalize the fundamental transition functions that humans use with language.
When provided with sufficient context (state), these models can achieve human-level performance precisely because they have learned to replicate human computational patterns through exposure to vast collections of text---our recorded explanations, arguments, derivations, and other traces of human thought \cite{brown2020language,wei2022chain}.
This reasoning occurs without weight updates, suggesting the models have internalized generalizable computational patterns from their training data.
The success of techniques like chain-of-thought prompting demonstrates that these models can effectively replicate human-like problem-solving approaches when given appropriate frameworks for maintaining computational state \cite{wei2022emergent}.

The success of these methods depends not on teaching models new reasoning capabilities, but on structuring computation to leverage already-learned transition functions in a way that enables reliable state maintenance and access.

\section{Conclusion}

Our analysis reveals a fundamental pattern across scales: from molecular mechanisms to neural systems to artificial intelligence, advances in computational capability emerge from enhanced abilities to maintain and access state---just as human civilization itself advanced through increasingly sophisticated memory technologies.
The recent breakthroughs in language models demonstrate this principle---their power comes not from architectural complexity but from learning human computational patterns while maintaining sufficient context for recursive processing.
The recent emergence of systems like o3 \cite{guan2024deliberative} demonstrates this principle.
The model we present here suggests their capabilities arise not from architectural sophistication or alignment policy adherence, but from meeting the fundamental requirements for memory-enabled universal computation, even at substantial computational cost in current implementations.
This suggests a shift in focus for AI development: rather than pursuing larger models, progress requires attention to the memory mechanisms that enable recursive computation.
Evolution's elegant solutions, from CRISPR arrays to neural architectures, point toward fundamentally new approaches to reliable computation based on robust state maintenance rather than complex processing units.

\section*{Acknowledgments}
This work was supported by National Science Foundation PPoSS Award \#2118709, National Institutes of Health 1U01HG013760-01, and the University of Tennessee Center for Integrative and Translational Genomics.

\begingroup
\footnotesize
\bibliographystyle{unsrturl}
\bibliography{refs}

\begin{thebibliography}{10}

\bibitem{alberts2022molecular}
Bruce Alberts, Rebecca Heald, Alexander Johnson, David Morgan, Martin Raff,
  Keith Roberts, and Peter Walter.
\newblock {\em Molecular Biology of the Cell: Seventh International Student
  Edition with Registration Card}.
\newblock WW Norton \& Company, 2022.

\bibitem{turing1936computable}
A~Turing.
\newblock On computable numbers, with an application to the entscheidungs
  problem.
\newblock {\em Proceedings of the London Mathematical Society Series/2 (42)},
  pages 230--42, 1936.

\bibitem{merrill2023parallelism}
William Merrill and Ashish Sabharwal.
\newblock The parallelism tradeoff: Limitations of log-precision transformers.
\newblock {\em Transactions of the Association for Computational Linguistics},
  11:531--545, 2023.
\newblock \href {https://doi.org/10.1162/tacl_a_00562}
  {\path{doi:10.1162/tacl_a_00562}}.

\bibitem{peng2024limitations}
Binghui Peng, Srini Narayanan, and Christos Papadimitriou.
\newblock On limitations of the transformer architecture.
\newblock {\em COLM}, 2024.

\bibitem{cook2004universality}
Matthew Cook.
\newblock Universality in elementary cellular automata.
\newblock {\em Complex Systems}, 15(1):1--40, 2004.

\bibitem{dolan2013mov}
Stephen Dolan.
\newblock mov is turing-complete.
\newblock Cambridge University Computer Laboratory, July 2013.
\newblock URL: \url{https://www.cl.cam.ac.uk/~sd601/papers/mov.pdf}.

\bibitem{wang2023parallel}
Boya Wang, Siyuan~Stella Wang, Cameron Chalk, Andrew~D Ellington, and David
  Soloveichik.
\newblock Parallel molecular computation on digital data stored in dna.
\newblock {\em PNAS}, 120(37), 2023.
\newblock \href {https://doi.org/10.1073/pnas.2217330120}
  {\path{doi:10.1073/pnas.2217330120}}.

\bibitem{swamy1983space}
Sowmitri Swamy and John~E Savage.
\newblock Space-time tradeoffs for linear recursion.
\newblock {\em Mathematical Systems Theory}, 16:9--27, 1983.
\newblock \href {https://doi.org/10.1007/BF01744566}
  {\path{doi:10.1007/BF01744566}}.

\bibitem{boyle2024memory}
Elette Boyle, Ilan Komargodski, and Neekon Vafa.
\newblock Memory checking requires logarithmic overhead.
\newblock In {\em Proceedings of the 56th Annual {ACM} Symposium on Theory of
  Computing (STOC '24)}, pages 1712--1723, New York, NY, USA, June 2024. {ACM}.
\newblock \href {https://doi.org/10.1145/3618260.3649686}
  {\path{doi:10.1145/3618260.3649686}}.

\bibitem{sipser1996introduction}
Michael Sipser.
\newblock Introduction to the theory of computation.
\newblock {\em ACM Sigact News}, 27(1):27--29, 1996.

\bibitem{savage1994space}
John~E Savage.
\newblock Space-time tradeoffs.
\newblock {\em Models of Computation}, 1994.

\bibitem{bennett1989time}
Charles~H Bennett.
\newblock Time/space trade-offs for reversible computation.
\newblock {\em SIAM Journal on Computing}, 18(4):766--776, 1989.

\bibitem{deutsch1995universality}
David Deutsch, Adriano Barenco, and Artur Ekert.
\newblock Universality in quantum computation.
\newblock {\em arXiv:quant-ph/9505018}, 1995.

\bibitem{liskiewicz1994complexity}
Maciej Liskiewicz and R{\"u}diger Reischuk.
\newblock On the complexity of space bounded computations.
\newblock {\em Theoretical Computer Science}, 147:1--45, 1994.

\bibitem{parzych2024memory}
Garrett Parzych and Joshua~J Daymude.
\newblock Memory lower bounds and impossibility results for anonymous dynamic
  broadcast.
\newblock {\em arXiv:2407.09714}, 2024.

\bibitem{hhan2024new}
Minki Hhan.
\newblock A new approach to generic lower bounds: Classical/quantum mdl,
  quantum factoring, and more.
\newblock {\em arXiv:2024.268}, 2024.

\bibitem{vonkorff2019molecular}
Modest von Korff and Thomas Sander.
\newblock Molecular complexity calculated by fractal dimension.
\newblock {\em Scientific Reports}, 9:1746, 2019.
\newblock \href {https://doi.org/10.1038/s41598-018-37253-8}
  {\path{doi:10.1038/s41598-018-37253-8}}.

\bibitem{martini2015information}
Markus Martini, Marco~R Furtner, Thomas Maran, and Pierre Sachse.
\newblock Information maintenance in working memory: An integrated presentation
  of cognitive and neural concepts.
\newblock {\em Frontiers in Systems Neuroscience}, 9:104, 2015.
\newblock \href {https://doi.org/10.3389/fnsys.2015.00104}
  {\path{doi:10.3389/fnsys.2015.00104}}.

\bibitem{quentin2019differential}
Romain Quentin et~al.
\newblock Differential brain mechanisms of selection and maintenance of
  information during working memory.
\newblock {\em Journal of Neuroscience}, 39(19):3728--3740, 2019.
\newblock \href {https://doi.org/10.1523/JNEUROSCI.2764-18.2019}
  {\path{doi:10.1523/JNEUROSCI.2764-18.2019}}.

\bibitem{bruno2022epigenetic}
Simone Bruno, Ruth~J Williams, and Domitilla Del~Vecchio.
\newblock Epigenetic cell memory: The gene's inner chromatin modification
  circuit.
\newblock {\em PLoS Computational Biology}, 18(4), 2022.
\newblock \href {https://doi.org/10.1371/journal.pcbi.1009961}
  {\path{doi:10.1371/journal.pcbi.1009961}}.

\bibitem{turing1952chemical}
A.~M. Turing.
\newblock The chemical basis of morphogenesis.
\newblock {\em Philosophical Transactions of the Royal Society of London.
  Series B, Biological Sciences}, 237(641):37--72, 1952.

\bibitem{barrett2019analyzing}
David~GT Barrett, Ari~S Morcos, and Jakob~H Macke.
\newblock Analyzing biological and artificial neural networks: challenges with
  opportunities for synergy?
\newblock {\em Current Opinion in Neurobiology}, 55:55--64, 2019.
\newblock \href {https://doi.org/10.1016/j.conb.2019.01.007}
  {\path{doi:10.1016/j.conb.2019.01.007}}.

\bibitem{bengio2017deep}
Yoshua Bengio, Ian Goodfellow, and Aaron Courville.
\newblock {\em Deep learning}, volume~1.
\newblock MIT press Cambridge, MA, USA, 2017.

\bibitem{gu2021efficiently}
Albert Gu, Karan Goel, and Christopher R{\'e}.
\newblock Efficiently modeling long sequences with structured state spaces.
\newblock {\em arXiv preprint arXiv:2111.00396}, 2021.

\bibitem{gu2023mamba}
Albert Gu and Tri Dao.
\newblock Mamba: Linear-time sequence modeling with selective state spaces.
\newblock {\em arXiv preprint arXiv:2312.00752}, 2023.

\bibitem{dao2024transformers}
Tri Dao and Albert Gu.
\newblock Transformers are ssms: Generalized models and efficient algorithms
  through structured state space duality.
\newblock {\em arXiv preprint arXiv:2405.21060}, 2024.

\bibitem{nguyen2024hyenadna}
Eric Nguyen, Michael Poli, Marjan Faizi, Armin Thomas, Michael Wornow, Callum
  Birch-Sykes, Stefano Massaroli, Aman Patel, Clayton Rabideau, Yoshua Bengio,
  et~al.
\newblock Hyenadna: Long-range genomic sequence modeling at single nucleotide
  resolution.
\newblock {\em Advances in neural information processing systems}, 36, 2024.

\bibitem{feng2024rnns}
Leo Feng, Frederick Tung, Mohamed~Osama Ahmed, Yoshua Bengio, and Hossein
  Hajimirsadeghi.
\newblock Were rnns all we needed?
\newblock {\em arXiv preprint arXiv:2410.01201}, 2024.

\bibitem{cai2024efficient}
Haoxiao Cai, Xiaoran Zhang, Rong Qiao, Xiaowo Wang, and Lei Wei.
\newblock Efficient computation by molecular competition networks.
\newblock {\em Physical Review Research}, 6:033208, 2024.

\bibitem{fu2023scgrn}
Shuhua Fu, Anqi Wang, and Kin~Fai Au.
\newblock scgrn: a comprehensive single-cell gene regulatory network platform.
\newblock {\em Nucleic Acids Research}, 52(D1):D293--D303, 2023.
\newblock \href {https://doi.org/10.1093/nar/gkad885}
  {\path{doi:10.1093/nar/gkad885}}.

\bibitem{milo2013protein}
Ron Milo.
\newblock What is the total number of protein molecules per cell volume? a call
  to rethink some published values.
\newblock {\em BioEssays}, 35(12):1050--1055, 2013.
\newblock \href {https://doi.org/10.1002/bies.201300066}
  {\path{doi:10.1002/bies.201300066}}.

\bibitem{zhang2023molecular}
Roland Zhang, Zach Drew, and Jeremy Jones.
\newblock Molecular tumbling rate effects on t1 and t2.
\newblock {\em Radiopaedia.org}, 2023.
\newblock URL: \url{https://radiopaedia.org/articles/60742}, \href
  {https://doi.org/10.53347/rID-60742} {\path{doi:10.53347/rID-60742}}.

\bibitem{merrill2024}
William Merrill and Ashish Sabharwal.
\newblock The expressive power of transformers with chain of thought.
\newblock {\em ICLR}, 2024.

\bibitem{kim2019crispr}
Jenny~G Kim, Sandra Garrett, Yunzhou Wei, Brenton~R Graveley, and Michael~P
  Terns.
\newblock Crispr dna elements controlling site-specific spacer integration and
  proper repeat length by a type ii crispr--cas system.
\newblock {\em Nucleic Acids Research}, 47(16):8632--8648, 2019.

\bibitem{sadremomtaz2023digital}
Afsaneh Sadremomtaz, Robert~F Glass, Jorge~Eduardo Guerrero, Dennis~R
  LaJeunesse, Eric~A Josephs, and Reza Zadegan.
\newblock Digital data storage on dna tape using crispr base editors.
\newblock {\em Nature Communications}, 14(1):6497, 2023.
\newblock \href {https://doi.org/10.1038/s41467-023-42223-4}
  {\path{doi:10.1038/s41467-023-42223-4}}.

\bibitem{pastor2020computation}
David Pastor-Escuredo and Juan~C del {\'A}lamo.
\newblock How computation is helping unravel the dynamics of morphogenesis.
\newblock {\em Frontiers in Physics}, 8:31, 2020.
\newblock \href {https://doi.org/10.3389/fphy.2020.00031}
  {\path{doi:10.3389/fphy.2020.00031}}.

\bibitem{lobo2012towards}
Daniel Lobo, Taylor~J Malone, and Michael Levin.
\newblock Towards a bioinformatics of patterning: a computational approach to
  understanding regulative morphogenesis.
\newblock {\em Biology Open}, 2(2):156--169, 2012.
\newblock \href {https://doi.org/10.1242/bio.20123400}
  {\path{doi:10.1242/bio.20123400}}.

\bibitem{sheth2017multiplex}
Ravi~U Sheth, Sung~Sun Yim, Felix~L Wu, and Harris~H Wang.
\newblock Multiplex recording of cellular events over time on crispr biological
  tape.
\newblock {\em Science}, 358(6369):1457--1461, 2017.
\newblock \href {https://doi.org/10.1126/science.aao0958}
  {\path{doi:10.1126/science.aao0958}}.

\bibitem{dickson2024trust}
Ben Dickson.
\newblock How far can you trust chain-of-thought prompting?
\newblock {\em BD Tech Talks}, 2024.

\bibitem{ahn2024recursive}
Jinwoo Ahn and Kyuseung Shin.
\newblock Recursive chain-of-feedback prevents performance degradation from
  redundant prompting.
\newblock {\em arXiv:2402.02648}, 2024.

\bibitem{openai2024o1}
{OpenAI}.
\newblock Openai o1 system card.
\newblock Technical report, OpenAI, December 2024.

\bibitem{liu2024mind}
Ryan Liu, Jiayi Geng, Addison~J Wu, Ilia Sucholutsky, Tania Lombrozo, and
  Thomas~L Griffiths.
\newblock Mind your step (by step): Chain-of-thought can reduce performance on
  tasks where thinking makes humans worse.
\newblock {\em arXiv preprint arXiv:2410.21333}, 2024.

\bibitem{shallue2019measuring}
Chris Shallue, Jaehoon Lee, Joe Antognini, Jascha Sohl-Dickstein, Roy Frostig,
  and George Dahl.
\newblock Measuring the effects of data parallelism in neural network training.
\newblock {\em Google Research Blog}, 2019.

\bibitem{chollet2024arc}
Fran\c{c}ois Chollet, Mike Knoop, Gregory Kamradt, and Bryan Landers.
\newblock Arc prize 2024: Technical report.
\newblock Technical report, ARC Prize, December 2024.

\bibitem{chollet2024o3}
François Chollet.
\newblock Openai o3 breakthrough high score on arc-agi-pub, 12 2024.
\newblock URL: \url{https://arcprize.org/blog/2024/o3-breakthrough}.

\bibitem{openai2024o3blog}
{OpenAI}.
\newblock Deliberative alignment: reasoning enables safer language models, 12
  2024.
\newblock URL: \url{https://openai.com/index/deliberative-alignment/}.

\bibitem{guan2024deliberative}
Melody~Y Guan, Manas Joglekar, Eric Wallace, Saachi Jain, Boaz Barak, et~al.
\newblock Deliberative alignment: Reasoning enables safer language models.
\newblock Technical report, 12 2024.
\newblock URL:
  \url{https://assets.ctfassets.net/kftzwdyauwt9/4pNYAZteAQXWtloDdANQ7L/0aedc43a8f2d1e5c71c5e114d287593f/OpenAI_Deliberative-Alignment-Reasoning-Enables-Safer_Language-Models_122024_3.pdf}.

\bibitem{openai2024learning}
{OpenAI}.
\newblock Learning to reason with {LLMs}, 9 2024.
\newblock URL: \url{https://openai.com/index/learning-to-reason-with-llms/}.

\bibitem{willison2024o3}
Simon Willison.
\newblock Openai o3 and the future of language models, 12 2024.
\newblock URL: \url{https://simonwillison.net/2024/Dec/21/o3}.

\bibitem{jung2020new}
Jaewoon Jung, Chigusa Kobayashi, Kento Kasahara, et~al.
\newblock New parallel computing algorithm of molecular dynamics for extremely
  huge scale biological systems.
\newblock {\em Journal of Computational Chemistry}, 42(4):231--241, 2020.
\newblock \href {https://doi.org/10.1002/jcc.26450}
  {\path{doi:10.1002/jcc.26450}}.

\bibitem{zhu2024overcoming}
Ruomin Zhu et~al.
\newblock Overcoming memory constraints in quantum circuit simulation with a
  high-fidelity compression framework.
\newblock {\em arXiv:2410.14088}, 2024.

\bibitem{espinosa2024molecular}
Menc{\'i}a Espinosa-Mart{\'i}nez, Mar{\'i}a Alc{\'a}zar-Fabra, and David
  Landeira.
\newblock The molecular basis of cell memory in mammals: The epigenetic cycle.
\newblock {\em Science Advances}, 10(8), 2024.
\newblock \href {https://doi.org/10.1126/sciadv.adl3188}
  {\path{doi:10.1126/sciadv.adl3188}}.

\bibitem{wei2022chain}
Jason Wei, Xuezhi Wang, Dale Schuurmans, Maarten Bosma, Brian Ichter, Fei Xia,
  Ed~Chi, Quoc Le, and Denny Zhou.
\newblock Chain-of-thought prompting elicits reasoning in large language
  models.
\newblock In {\em NeurIPS}, 2022.

\bibitem{schuurmans2024autoregressive}
Dale Schuurmans, Hanjun Dai, and Francesco Zanini.
\newblock Autoregressive large language models are computationally universal.
\newblock {\em arXiv preprint arXiv:2410.03170}, 2024.

\bibitem{brown2020language}
Tom~B Brown, Benjamin Mann, Nick Ryder, Melanie Subbiah, Jared Kaplan, Prafulla
  Dhariwal, Arvind Neelakantan, et~al.
\newblock Language models are few-shot learners.
\newblock {\em arXiv preprint arXiv:2005.14165}, 2020.

\bibitem{wei2022emergent}
Jason Wei, Yi~Tay, Rishi Bommasani, Colin Raffel, Barret Zoph, et~al.
\newblock Emergent abilities of large language models.
\newblock {\em Transactions on Machine Learning Research}, 2022.

\end{thebibliography}
\endgroup

\end{document}